\documentclass[journal]{IEEEtran}
\ifCLASSINFOpdf
\else
\fi
\usepackage{amsmath}
\usepackage{graphicx}
\usepackage{amssymb}
\usepackage{colortbl}
\usepackage{arydshln}
\usepackage{stfloats}
\usepackage{algorithm}
\usepackage{algorithmicx}
\usepackage{algpseudocode}
\usepackage{multirow}
\usepackage{extarrows}
\usepackage{mathrsfs}
\usepackage[sort]{cite}

\newtheorem{lemma}{Lemma}
\newtheorem{proposition}{Proposition}
\newtheorem{definition}{Definition}
\newtheorem{remark}{Remark}
\newtheorem{assumption}{Assumption}
\newtheorem{problem}{Problem}

\newtheorem{example}{Example}

\begin{document}
	%
	\title{Output Reachable Set Estimation and Verification for Multi-Layer Neural Networks}
	%
	%
	%
	
	\author{Weiming~Xiang, \emph{Senior Member, IEEE}, Hoang-Dung Tran  \emph{Member, IEEE}, and Taylor T. Johnson  \emph{Member, IEEE}
		\thanks{The material presented in this paper is based upon work supported by
			the National Science Foundation (NSF) under grant numbers CNS 1464311
			and 1713253, and SHF 1527398 and 1736323, and the Air Force Office of
			Scientific Research (AFOSR) under contract numbers FA9550-15-1-0258,
			FA9550-16-1-0246, and FA9550-18-1-0122. The U.S. government is
			authorized to reproduce and distribute reprints for Governmental
			purposes notwithstanding any copyright notation thereon. Any opinions,
			findings, and conclusions or recommendations expressed in this
			publication are those of the authors and do not necessarily reflect
			the views of AFOSR or NSF. 
			
			Authors are with the Department of Electrical Engineering and Computer Science, Vanderbilt University, Nashville, TN 37212 USA. Email: Weiming Xiang (xiangwming@gmail.com), Hoang-Dung Tran (trhoangdung@gmail.com), Taylor T. Johnson (taylor.johnson@gmail.com).}
	}
	
	%
	%

	\markboth{IEEE TRANSACTIONS ON NEURAL NETWORKS AND LEARNING SYSTEMS, VOL. XX, NO. XX, XX XXXX}%
	{}
	%



	\maketitle

	\begin{abstract}
		\boldmath
		In this paper, the output reachable estimation and safety verification problems for multi-layer perceptron neural networks are addressed. First, a conception called maximum sensitivity in introduced and, for a class of  multi-layer perceptrons whose activation functions are monotonic functions, the maximum sensitivity can be computed via solving  convex optimization problems. Then, using a simulation-based method, the output reachable set estimation problem for neural networks is formulated into a chain of optimization problems. Finally, an automated safety verification is developed based on the output reachable set estimation result. An application to the safety verification for a robotic arm model with two joints is presented to show the effectiveness of proposed approaches.
		
	\end{abstract}


	%
	\IEEEpeerreviewmaketitle
	
	\section{Introduction}
	Artificial neural networks have been widely used in machine learning systems. Applications include adaptive control \cite{hunt1992neural,ge1999adaptive,wang2016combined,wu2014exponential}, pattern recognition \cite{schmidhuber2015deep,lawrence1997face}, game playing \cite{silver2016mastering}, autonomous vehicles \cite{bojarski2016end}, and many others. Though neural networks have been showing the effectiveness and powerful ability in dealing with complex problems, they are confined to systems which comply only to the lowest safety integrity levels since, in most of the time, a neural network is viewed as a \emph{black box} without effective methods to assure safety specifications for its outputs. Verifying neural networks is a hard problem, even simple properties about them have been proven NP-complete
	problems \cite{katz2017reluplex}. The difficulties mainly come from the presence of activation functions and the complex structures, making neural networks large-scale, nonlinear, non-convex and thus incomprehensible to humans. Until now, only few results have been reported for verifying neural networks. The verification for feed-forward multi-layer neural networks is investigated based on \emph{Satisfiability Modulo Theory} (SMT) in \cite{huang2017safety,pulina2012challenging}. In \cite{pulina2010abstraction} an Abstraction-Refinement approach is proposed. In \cite{katz2017reluplex,xiang2017reachable_arxiv}, a specific kind of activation functions called \emph{Rectified Linear Unit} is considered for verification of neural networks. Additionally, some recent reachable set estimation results are reported for neural networks \cite{xu2017reachable,zuo2014non,thuan2016reachable}, these results that are based on Lyapunov functions analogous to stability \cite {xiang2017robust,xiang2018parameter,xiang2017stability,xiang2016necessary} and reachability analysis of dynamical systems \cite{xiang2017output,xiang2017reachable}, have potentials to be further extended to safety verification.

	In this work, we shall focus on a class of neural networks called \emph{Multi-Layer Perceptron} (MLP).  Due to the complex structure, manual reasoning for an MLP is impossible. Inspired by some simulation-based ideas for verification \cite{duggirala2015c2e2,fan2016automatic,bak2017hylaa}, the information collected from a finitely many simulations will be exploited to estimate the output reachable set of an MLP and, furthermore, to do safety verification. To bridge the gap between the finitely many simulations and the output set generated from a bounded input set which  essentially  includes infinite number of inputs, a conception called maximum sensitivity is introduced to characterize the maximum deviation of the output subject to a bounded disturbance around a nominal input. By formulating a chain of optimizations, the maximum sensitivity for an MLP can be computed in a layer-by-layer manner. Then,  an exhaustive search of the input set is enabled by a discretization of input space to achieve an estimation of output reachable set which consists of a union of reachtubes. Finally, by the merit of reachable set estimation, the safety verification for an MLP can be done via checking the existence of intersections between the estimated reachable set and unsafe regions. The main benefits of our approach are that there are very few restrictions on the activation functions except for the monotonicity which is satisfied by a variety of activation functions, and also no requirement on the bounded input sets. All these advantages are coming from the simulation-based nature of our approach. 
	
	The remainder of this paper is organized as follows. Preliminaries and problem formulation are given in Section II. The maximum sensitivity analysis for an MLP is studied in Section III. Output reachable set estimation and safety verification results are given in Section IV. An application to robotic arms is provided in Section V and Conclusions are presented in Section VI.

	\section{Preliminaries and Problem Formulation} 
	\subsection{Multi-Layer Neural Networks}
	A neural network consists of a number of interconnected neurons. Each neuron is a simple processing element that responds to the weighted inputs it received from other neurons. In this paper, we consider the most popular and general feedforward neural networks called the Multi-Layer Perceptron (MLP). Generally, an MLP consists of three typical classes of layers: An input layer, that
	serves to pass the input vector to the network, hidden layers of computation neurons, and
	an output layer composed of at least a computation neuron to produce the output vector.
	The action of a neuron depends on its activation function, which is described as 
	\begin{align}
	y_i = f\left(\sum\nolimits_{j=1}^{n}\omega_{ij} x_j + \theta_i\right)
	\end{align}
	where $x_j$ is the $j$th input of the $i$th neuron, $\omega_{ij}$ is the weight from the $j$th input to the $i$th neuron, $\theta_i$ is called the bias of the $i$th neuron, $y_i$ is the output of the $i$th neuron, $f(\cdot)$ is the activation function. The activation function is a nonlinear function  describing the reaction of $i$th neuron with inputs $x_j(t)$, $j=1,\cdots,n$. Typical activation functions include rectified linear unit, logistic, tanh, exponential linear unit, linear functions, for instance. In this work, our approach aims at dealing with the most of activation functions regardless of their specific forms, only the following monotonic assumption needs to be satisfied.
	
	\begin{assumption}\label{assumption_1}
		For any $x_1 \le x_2$, the activation function satisfies $f(x_1) \le f(x_2)$. 
	\end{assumption}
	\begin{remark}
		Assumption \ref{assumption_1} is a common property that can be satisfied by a variety of activation functions. For example, it is easy to verify that the most commonly used logistic function $f(x)=1/(1+e^{-x})$
		satisfies	Assumption \ref{assumption_1}.
	\end{remark}
	
	An MLP has multiple layers,  each layer $\ell$, $1 \le \ell \le L $, has $n^{[\ell]}$ neurons.  In particular, layer $\ell =0$ is used to denote the input layer and $n^{[0]}$ stands for the number of inputs in the rest of this paper, and $n^{[L]}$ stands for the last layer, that is the output layer. For a neuron $i$, $1 \le i \le n^{[\ell]}$ in layer $\ell$, the corresponding input vector is denoted by $\mathbf{x}^{[\ell]}$ and the weight matrix is 
	$
	\mathbf{W}^{[\ell]} = [\boldsymbol{\omega}_{1}^{[\ell]},\ldots,\boldsymbol{\omega}_{n^{[\ell]}}^{[\ell]}]^{\top}
	$,
	where $\boldsymbol{\omega}_{i}^{[\ell]}$ is the weight vector. The bias vector for layer $\ell$ is
	$ \boldsymbol{\theta}^{[\ell]}=[\theta_1^{[\ell]},\ldots,\theta_{n^{[\ell]}}^{[\ell]}]^{\top}
	$. 
	
	The output vector of layer $\ell$ can be expressed as 
	\begin{equation*}
	\mathbf{y}^{[\ell]}=f_{\ell}(\mathbf{W}^{[\ell]}\mathbf{x}^{[\ell]}+\boldsymbol{\theta}^{[\ell]})
	\end{equation*} 
	where $f_{\ell}(\cdot)$ is the activation function for layer $\ell$.
	
	For an MLP, the output of $\ell-1$ layer is the input of $\ell$ layer. The mapping from the input $\mathbf{x}^{[0]}$ of input layer  to the output  $\mathbf{y}^{[L]}$ of output layer stands for the input-output relation of the MLP, denoted by
	\begin{equation}\label{NN}
	\mathbf{y}^{[L]} = F (\mathbf{x}^{[0]})
	\end{equation}    
	where $F(\cdot) \triangleq f_L  \circ f_{L - 1}  \circ  \cdots  \circ f_1(\cdot) $.
	
	According to the \emph{Universal Approximation Theorem} \cite{hornik1989multilayer}, it guarantees that, in
	principle, such an MLP in the form of (\ref{NN}), namely the function $F(\cdot)$, is able to approximate any nonlinear real-valued function. Despite the impressive ability of approximating  functions,  much complexities represent in predicting the output behaviors of an MLP. In most of real applications, an MLP is usually viewed as a \emph{black box} to generate  a desirable output with respect to a given input. However, regarding  property verifications such as  safety verification, it has been observed that even a well-trained neural network can react in unexpected and incorrect ways to even slight perturbations of their inputs, which could result in unsafe systems. Thus, the output reachable set estimation of an MLP, which is able to cover all possible values of outputs, is necessary for the safety verification of an MLP and draw a safe or unsafe conclusion for an MLP.

	\subsection{Problem Formulation}
	Given an input set $\mathcal{X}$, the output reachable set of neural network (\ref{NN}) is stated by the following definition.
	\begin{definition} \label{reachable_set}
		Given an MLP in the form of (\ref{NN}) and an input set $\mathcal{X}$, the output reachable set of (\ref{NN}) is defined as 
		\begin{equation}
		\mathcal{Y} \triangleq \{\mathbf{y}^{[L]}  \mid \mathbf{y}^{[L]} = F (\mathbf{x}^{[0]}),~\mathbf{x}^{[0]} \in \mathcal{X}\} .
		\end{equation}
	\end{definition}
	
	Since MLPs are often large, nonlinear, and non-convex, it is extremely difficult to compute the exact output reachable set $\mathcal{Y}$ for an MLP. Rather than directly computing the exact output reachable set for an MLP, a more practical and feasible way is to derive an over-approximation of $\mathcal{Y}$, which is called output reachable set estimation.
	
	\begin{definition}\label{reachable_estimation}
		A set  $\tilde{\mathcal{Y}} $ is called an output reachable set estimation of MLP (\ref{NN}), if $\mathcal{Y} \subseteq \tilde{\mathcal{Y}}$ holds, where $\mathcal{Y}$ is the output reachable set of MLP (\ref{NN}). 
	\end{definition}
	
	Based on Definition \ref{reachable_estimation}, the problem of output reachable set estimation for an MLP is given as below.
	
	\begin{problem} \label{problem1}
		Given a bounded input set $\mathcal{X}$ and an MLP described by (\ref{NN}), how to find a set $\tilde{\mathcal{Y}} $ such that $\mathcal{Y} \subseteq \tilde{\mathcal{Y}}$, and make the estimation set $\tilde{\mathcal{Y}}$ as small as possible\footnote{	For a set $\mathcal{Y}$, its over-approximation $\tilde{\mathcal{Y}}_1$ is smaller than another over-approximation $\tilde{\mathcal{Y}}_2$ if 
			$
			d_H(\tilde{\mathcal{Y}}_1,\mathcal{Y}) < d_H(\tilde{\mathcal{Y}}_2,\mathcal{Y})
			$ holds,
			where $d_H$ stands for the Hausdorff distance.}? 
	\end{problem}
	
	
	In this work, we will focus on the safety verification for neural networks. The safety specification for outputs is expressed by a set defined in the output space, describing the safety requirement. 
	
	\begin{definition}
		Safety specification $\mathcal{S}$ of an MLP
		formalizes the safety requirements for  output $\mathbf{y}^{[L]}$ of MLP $\mathbf{y}^{[L]}=F(\mathbf{x}^{[0]})$, and is a predicate over  output $\mathbf{y}^{[L]}$ of MLP. The MLP is safe if and only if the following condition is satisfied:
		\begin{equation}\label{safety}
		\mathcal{Y} \cap \neg \mathcal{S} = \emptyset
		\end{equation}
		where $\neg$ is the symbol for logical negation.
	\end{definition}
	
	Therefore, the safety verification problem for an MLP is stated as follows.
	\begin{problem}\label{problem2}
		Given a bounded input set $\mathcal{X}$, an MLP in the form of (\ref{NN}) and a safety specification $\mathcal{S}$, how to check if condition (\ref{safety}) is satisfied? 
	\end{problem}
	
	Before ending this section, a lemma is presented to show that the safety verification of an MLP can be relaxed by checking with the over-approximation of the output reachable set. 
	
	\begin{lemma}\label{lemma1}
		Consider an MLP in the form of (\ref{NN}), an output reachable set estimation $\mathcal{Y}\subseteq\tilde{\mathcal{Y}}$ and a safety specification $\mathcal{S}$, the MLP is safe if the following condition is satisfied
		\begin{equation}\label{estimate_safety}
		\tilde{\mathcal{Y}} \cap \neg \mathcal{S} = \emptyset .
		\end{equation}
	\end{lemma}
	\begin{proof}
		Since $\mathcal{Y}\subseteq\tilde{\mathcal{Y}}$, (\ref{estimate_safety}) directly leads to $\mathcal{Y} \cap \neg \mathcal{S} = \emptyset$. The proof is complete.
	\end{proof}
	
	Lemma \ref{lemma1} implies that it is sufficient to use the estimated output reachable set for the safety verification of an MLP, thus the solution of Problem \ref{problem1} is also the key to solve Problem \ref{problem2}. 
	
	\section{Maximum Sensitivity for Neural Networks}
	Due to the complex structure and nonlinearities in activation functions, estimating the output reachable sets of MLPs represents much difficulties if only using analytical methods. One possible way to circumvent those difficulties is to employ the information produced by a finite number of simulations. As well known, the finitely many simulations generated from input set $\mathcal{X}$ are incomplete to characterize output set $\mathcal{Y}$, a conception called maximum sensitivity is introduced to bridge the gap between simulations and output reachable set estimations of MLPs.   
	
	\begin{definition} \label{def4}
		Given an MLP $\mathbf{y}^{[L]} = F (\mathbf{x}^{[0]})$, an input $\mathbf{x}^{[0]}$ and disturbances $\Delta\mathbf{x}^{[0]}$ satisfying $\left\|\Delta\mathbf{x}^{[0]} \right\|\le\delta$, the maximum sensitivity of the MLP with input error $\delta$ at $\mathbf{x}^{[0]}$ is defined by 
		\begin{align}
		\epsilon_F(\mathbf{x}^{[0]},\delta)\triangleq\inf\{\epsilon:~\left\|\Delta\mathbf{y}^{[L]} \right\|\le\epsilon,&  \nonumber
		\\
		\mathrm{where}~\mathbf{y}^{[L]} = F (\mathbf{x}^{[0]})
		~&\mathrm{and}~\left\|\Delta\mathbf{x}^{[0]} \right\|\le\delta\} \label{sensitivity}
		\end{align}
	\end{definition}
	
	\begin{remark}
		In some previous articles as \cite{zeng2001sensitivity,zeng2003quantified}, the sensitivity for  neural networks is defined as the mathematical expectation of output deviations due to input and weight deviations with respect to overall input and weight values in a given continuous interval. The sensitivity in the average point of view works well for learning algorithm improvement \cite{xi2013architecture}, weight selection \cite{piche1995selection}, architecture construction \cite{shi2005sensitivity}, for instance. However, it cannot be used for safety verification due to the concern of soundness. In this paper, the maximum sensitivity is introduced to measure the maximum deviation of outputs, which is caused by the bounded disturbances around the nominal input $\mathbf{x}^{[0]}$.
	\end{remark}
	
	Due to the multiple layer structure, we are going to develop a layer-by-layer method to compute the maximum sensitivity defined by (\ref{sensitivity}).
	First, we consider a single layer $\ell$. According to Definition \ref{def4}, the maximum sensitivity for layer $\ell$, which is denoted by $\epsilon(\mathbf{x}^{[\ell]},\delta^{[\ell]})$ at $\mathbf{x}^{[\ell]}$, can be computed by	
	\begin{align}
	&\max~\epsilon(\mathbf{x}^{[\ell]},\delta^{[\ell]}) \nonumber
	\\
	\mathrm{s.t.}~& 
	\epsilon(\mathbf{x}^{[\ell]},\delta^{[\ell]})=\left\|f_{\ell}(\mathbf{W}^{[\ell]}(\mathbf{x}^{[\ell]}+\Delta \mathbf{x}^{[\ell]})+\boldsymbol{\theta}^{[\ell]}) - \mathbf{y}^{[\ell]} \right\| \nonumber
	\\
	& \mathbf{y}^{[\ell]}=f_{\ell}(\mathbf{W}^{[\ell]}\mathbf{x}^{[\ell]}+\boldsymbol{\theta}^{[\ell]}) 
	\nonumber
	\\
	&  \left\|\Delta\mathbf{x}^{[\ell]} \right\|\le\delta^{[\ell]} . \label{opt_1}
	\end{align}
	
	In the rest of paper, the norm $\left\| \cdot \right\|$ is considered the infinity norm, that is $\left\|\cdot\right\|_{\infty}$. By the definition of $\left\|\cdot\right\|_{\infty}$ and monotonicity assumption in Assumption \ref{assumption_1}, the optimal solution $\Delta\mathbf{x}^{[\ell]}_{\mathrm{opt}}$ of (\ref{opt_1}) can be found by running the following set of optimization problems.
	
	To find the optimal solution of (\ref{opt_1}) for layer $\ell$, we start from the neuron $i$ in layer $\ell$, the following two convex optimizations can be set up
	\begin{align}
	&\max~\beta_{i}^{[\ell]} \nonumber
	\\
	\mathrm{s.t.}~~& 
	\beta_i^{[\ell]}=(\boldsymbol{\omega}_i^{[\ell]})^{\top}(\mathbf{x}^{[\ell]}+\Delta \mathbf{x}^{[\ell]})+\theta^{[\ell]} \nonumber
	\\
	&  \left\|\Delta\mathbf{x}^{[\ell]} \right\|_{\infty}\le\delta^{[\ell]} \label{opt_2}
	\end{align}
	and 		
	\begin{align}
	&\min~\beta_{i}^{[\ell]} \nonumber
	\\
	\mathrm{s.t.}~~& 
	\beta_i^{[\ell]}=(\boldsymbol{\omega}_i^{[\ell]})^{\top}(\mathbf{x}^{[\ell]}+\Delta \mathbf{x}^{[\ell]})+\theta^{[\ell]} \nonumber
	\\
	&  \left\|\Delta\mathbf{x}^{[\ell]} \right\|_{\infty}\le\delta^{[\ell]} . \label{opt_3}
	\end{align}
	
	Then, due to the monotonicity, the following optimization problem defined over a finite set consisting of $\beta_{i,\max}^{[\ell]}$ and $\beta_{i,\min}^{[\ell]}$ obtained in (\ref{opt_2}) and (\ref{opt_3}) is formulated to compute the maximum absolute value of output of neuron $i$ in layer $\ell$ 
	\begin{align}
	&\max~\gamma_{i}^{[\ell]} \nonumber
	\\
	\mathrm{s.t.}~~& \gamma_{i}^{[\ell]}=\left|f_{\ell}(\beta_i^{[\ell]}) - f_{\ell}((\boldsymbol{\omega}_i^{[\ell]})^{\top}(\mathbf{x}^{[\ell]})+\theta^{[\ell]})\right| \nonumber
	\\
	&\beta_i^{[\ell]} \in \{\beta_{i,\min}^{[\ell]},~\beta_{i,\max}^{[\ell]} \} . \label{opt_4}
	\end{align}
	
	Finally, based on the maximum  absolute value of the output of neuron $i$ and because of the definition of infinity norm, we are ready to compute the maximum sensitivity of layer $\ell$ by picking out the largest value of $\gamma_i^{[\ell]}$ in layer $\ell$, that is
	\begin{align}
	&\max~\epsilon(\mathbf{x}^{[\ell]},\delta^{[\ell]})  \nonumber
	\\
	\mathrm{s.t.}~~& \epsilon(\mathbf{x}^{[\ell]},\delta^{[\ell]})  \in \{\gamma_1^{[\ell]},\ldots,\gamma_{n^{[\ell]}}^{[\ell]}\} . \label{opt_5}
	\end{align}
	
	In summary, the maximum sensitivity of a single layer $\ell$ can be computed through solving optimizations (\ref{opt_2})--(\ref{opt_5}) sequentially.
	
	\begin{proposition}
		Given a single layer $\ell$, the maximum sensitivity $\epsilon(\mathbf{x}^{[\ell]},\delta^{[\ell]})$ is the solution of (\ref{opt_5}) in which $\gamma_1^{[\ell]},\ldots,\gamma_{n^{[\ell]}}^{[\ell]}$ are solutions of (\ref{opt_4}) with $\beta_{i,\min}^{[\ell]},~\beta_{i,\max}^{[\ell]}$ being solutions of  (\ref{opt_2}) and (\ref{opt_3}). 
	\end{proposition}
	
	The above optimizations (\ref{opt_2})--(\ref{opt_5}) provide a way to compute the maximum sensitivity  for one layer. Then, for an MLP, we have $\mathbf{x}^{[\ell]} = \mathbf{y}^{[\ell-1]}$, $\ell = 1,\ldots,L$, so the output of each layer can be computed by iterating above optimizations with updated input $\mathbf{x}^{[\ell]} = \mathbf{y}^{[\ell-1]}$,  $\delta^{[\ell]} = \epsilon(\mathbf{x}^{[\ell-1]},\delta^{[\ell-1]})$, $\ell = 1,\ldots,L$. The maximum sensitivity of   neural network $\mathbf{y}^{[L]} = F (\mathbf{x}^{[0]})$ is the outcome of optimization (\ref{opt_5}) for output layer $L$, namely, $\epsilon_F(\mathbf{x}^{[0]},\delta) = \epsilon(\mathbf{x}^{[L]},\delta^{[L]})$. The layer-by-layer idea is illustrated in Fig. \ref{fig_1}, which shows the general idea of the computation process for multiple layer neural networks.
	
	In conclusion, the computation for the maximal sensitivity of an MLP is converted to a chain of optimization problems. Furthermore, the optimization problems (\ref{opt_2}), (\ref{opt_3}) are convex optimization problems which can be efficiently solved by existing tools such as \texttt{cvx}, \texttt{linprog} in Matlab. To be more efficient in computation without evaluating the objective function repeatedly, we can even pre-generate the expression of optimal solutions given the weight and bias of the neural network.  Optimizations (\ref{opt_4}), (\ref{opt_5}) only have finite elements to search for the optimum, which can be also computed efficiently. The algorithm for computing the maximum sensitivity of an MLP is given in Algorithm \ref{algorithm_1}.  
	\begin{figure}
		\begin{center}
			\includegraphics[width=8cm]{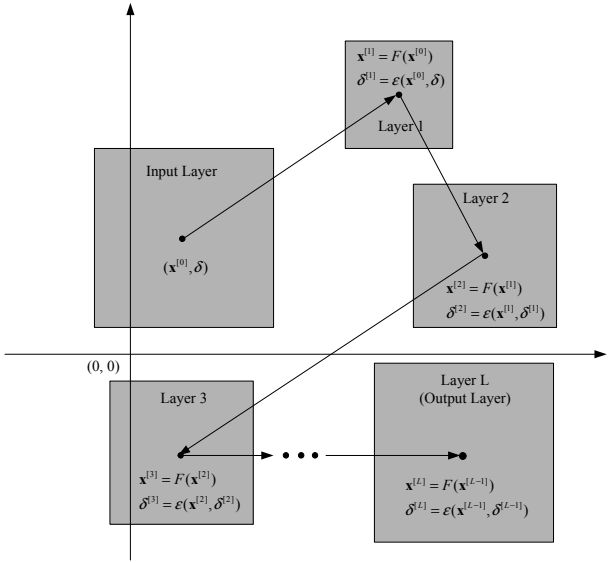}
			\caption{Illustration for computing maximum sensitivity for an MLP.}
			\label{fig_1}
		\end{center}
	\end{figure}
	
	\begin{algorithm}
		\caption{Maximum Sensitivity Computation Function for MLP} \label{algorithm_1}
		\begin{algorithmic}[1]
			\Require MLP $F$, input $\mathbf{x}^{[0]}$ and disturbance error $\delta$.
			\Ensure Maximum Sensitivity  $\epsilon_F(\mathbf{x}^{[0]},\delta) $. 
			\Function{MaxSensitivity}{$F$,~$\mathbf{x}^{[0]}$,~$\delta$}
			\State $\mathbf{x}^{[1]} \gets \mathbf{x}^{[0]}$; $\delta^{[1]}\gets\delta$
			\For {$\ell = 1: 1: L-1$}
			\State Solve (\ref{opt_2}), (\ref{opt_3}) to obtain $\beta_{i,\min}^{[\ell]}$, $\beta_{i,\max}^{[\ell]}$
			\State With $\beta_{i,\min}^{[\ell]}$, $\beta_{i,\max}^{[\ell]}$, solve  (\ref{opt_4}) to obtain $\gamma_i^{[\ell]}$
			\State With $\gamma_i^{[\ell]}$, solve (\ref{opt_5}) to obtain $\epsilon(\mathbf{x}^{[\ell]},\delta^{[\ell]})$
			\State  $\mathbf{x}^{[\ell+1]} \gets f_{\ell}(\mathbf{W}^{[\ell]}\mathbf{x}^{[\ell]}+\boldsymbol{\theta}^{[\ell]}) $;  $\delta^{[\ell+1]} \gets \epsilon(\mathbf{x}^{[\ell]},\delta^{[\ell]})$
			\EndFor
			\State With $\mathbf{x}^{[L]}$, $\delta^{[L]}$, solve  (\ref{opt_2})--(\ref{opt_5}) to obtain $\epsilon(\mathbf{x}^{[L]},\delta^{[L]})$
			\State $\epsilon_F(\mathbf{x}^{[0]},\delta) \gets \epsilon(\mathbf{x}^{[L]},\delta^{[L]})$
			
			\State \Return  $\epsilon_F(\mathbf{x}^{[0]},\delta)$
			\EndFunction
		\end{algorithmic}
	\end{algorithm}

	\section{Reachable Set Estimation and Verification}
	
	In previous section, the maximum sensitivity for an MLP can be computed via a chain of optimizations. The computation result actually can be viewed as a \emph{reachtube} for the inputs around nominal input $\mathbf{x}^{[0]}$, that are the inputs bounded in the tube $\left\|\Delta \mathbf{x}^{[0]}\right\|_{\infty}\le\delta$. This allows us to relate the individual simulation outputs to the output reachable set of an MLP. 
	
	First, the input space is discretized into lattices, which are described by
	\begin{equation}
	\mathcal{L}_i \triangleq \{\mathbf{x}^{[0]} \mid \left\|\mathbf{x}^{[0]}-\mathbf{x}_i^{[0]}\right\|_{\infty} \le \delta\}
	\end{equation}
	where $\mathbf{x}_i^{[0]}$ and $\delta$ are called the center and the radius of $\mathcal{L}_i$, respectively. The sets $\mathcal{L}_i$ satisfy $\mathcal{L}_i \cap \mathcal{L}_j = \{\mathbf{x}^{[0]} \mid \left\|\mathbf{x}^{[0]}-\mathbf{x}_i^{[0]}\right\|_{\infty} = \delta \wedge  \left\|\mathbf{x}^{[0]}-\mathbf{x}_j^{[0]}\right\|_{\infty} = \delta \}$ and $\bigcup_{i=1}^{\infty}{\mathcal{L}_i} =\mathbb{R}^{n \times n}$. Obviously, for any bounded set $\mathcal{X}$, we can find a finite number of $\mathcal{L}_i$ such that  $\mathcal{L}_i \bigcap \mathcal{X} \ne \emptyset$. The index set for  all $\mathcal{L}_i$ satisfying  $\mathcal{L}_i \cap \mathcal{X} \ne \emptyset$  is denoted by $\mathcal{I}$, so it can be obtained  that $\mathcal{X} \subseteq \bigcup_{i \in \mathcal{I}} \mathcal{L}_i$.  Explicitly, the lattices with a smaller radius $\delta$ are able to achieve a preciser approximation of bounded set $\mathcal{X}$ and, moreover, $\bigcup_{i \in \mathcal{I}} \mathcal{L}_i$ will exactly be $\mathcal{X}$ if radius $\delta \to 0$. The number of lattices is closely related to the dimension of the input space and radius chosen for discretization. Taking a unit box $\{\mathbf{x} \in \mathbb{R}^{n} \mid \left\|\mathbf{x}\right\| \le 1\}$ for example, the number of lattices with radius $\delta$ is 
	$
	{\lceil{1}/{2\delta} \rceil}^{n}
	$.  
	
	The first step is to derive all the lattices $\mathcal{L}_i$, $i \in \mathcal{I}$ for the input set $\mathcal{X}$ such that $\mathcal{L}_i \cap \mathcal{X} \ne \emptyset$, $\forall i \in \mathcal{I}$. 
	Then, based on the maximum sensitivity computation result, the output reachtube for each lattice $\mathcal{L}_i$ can be obtained by using Algorithm \ref{algorithm_1}.  
	Since $\mathcal{X} \subseteq \bigcup_{i \in \mathcal{I}} \mathcal{L}_i$, the union of output reachtubes of $\mathcal{L}_i$, $i \in \mathcal{I}$ includes all the possible outputs generated by the neural network with input set $\mathcal{X}$.  
	The following proposition is the main result in this work.
	
	\begin{proposition}\label{proposition_1}
		Given an MLP $\mathbf{y}^{[L]} = F(\mathbf{x}^{[0]})$, input set $\mathcal{X}$ and lattices $\mathcal{L}_i$, $i \in \mathcal{I}$ with centers $\mathbf{x}^{[0]}_i$ and radius $\delta$, and all the lattices satisfy $\mathcal{L}_i \cap \mathcal{X} \ne \emptyset$, $\forall i \in \mathcal{I}$, the output reachable set $\mathcal{Y}$ satisfies $\mathcal{Y} \subseteq \tilde{\mathcal{Y}} \triangleq \bigcup_{i \in \mathcal{I}} \tilde{\mathcal{Y}}_i$, where 
		\begin{equation}
		\tilde{\mathcal{Y}}_i \triangleq \{\mathbf{y}^{[L]}\mid \left\|\mathbf{y}^{[L]}-\mathbf{y}_i^{[L]}\right\|_{\infty} \le \epsilon_F(\mathbf{x}_i^{[0]},\delta),~\mathbf{y}_i^{[L]} = F(\mathbf{x}_i^{[0]})\}
		\end{equation} where $\epsilon_F(\mathbf{x}_i^{[0]},\delta)$ is computed by Algorithm \ref{algorithm_1}.
	\end{proposition}
	\begin{proof}
		Using Algorithm \ref{algorithm_1} for inputs within  lattice $\mathcal{L}_i$,  $\tilde{\mathcal{Y}}_i$ is the reachtube for $\mathcal{L}_i$ via the given MLP. Thus, the union of $\tilde{\mathcal{Y}}_i$, that is $\bigcup_{i \in \mathcal{I}} \tilde{\mathcal{Y}}_i$, is the output reachable set of $\bigcup_{i \in \mathcal{I}} \mathcal{L}_i$. Moreover, due to  $\mathcal{X} \subseteq \bigcup_{i \in \mathcal{I}} \mathcal{L}_i$, it directly implies that the output reachable set of $\mathcal{X} $ is a subset of $ \bigcup_{i \in \mathcal{I}} \tilde{\mathcal{Y}}_i$, that is 
		$
		\mathcal{Y} \subseteq \tilde{\mathcal{Y}} \triangleq \bigcup_{i \in \mathcal{I}} \tilde{\mathcal{Y}}_i
		$.
		The proof is complete.
	\end{proof}
	
	Based on Proposition \ref{proposition_1}, the output reachable set estimation involves the following two key steps: 
	\begin{enumerate}
		\item[(1)] Execute a finite number of simulations for MLP to get individual outputs $\mathbf{y}_i^{[L]}$ with respect to individual inputs $\mathbf{x}_i^{[0]}$. This can be done by simply generating the outputs with a finite number of inputs through the MLP as  $\mathbf{y}^{[L]}_i = F(\mathbf{x}^{[0]}_i)$. That is the main reason that our approach is called simulation-based.
		\item[(2)] Compute the maximum sensitivity for a finite number of lattices centered at $\mathbf{x}_i^{[0]}$, which can be solved by the \texttt{MaxSensitivity} function proposed in Algorithm \ref{algorithm_1}. This step is to produce the reachtubes based on the simulation results, and combine them for the reachable set estimation of outputs.
	\end{enumerate}
	
	The complete algorithm to perform the output reachable set estimation for an MLP is summarized in Algorithm \ref{algorithm_2}, and Example \ref{example_1} is provided to validate our approach.
	
	\begin{algorithm}
		\caption{Output Reachable Set Estimation for MLP} \label{algorithm_2}
		\begin{algorithmic}[1]
			\Require MLF $F$, input set $\mathcal{X}$.
			\Ensure Estimated output set $\tilde{\mathcal{Y}}$. 
			\Function{OutputReach}{$F$,~$\mathcal{X}$}
			
			\State Initialize $\mathcal{L}_i$, $i \in \mathcal{I}$, $\mathbf{x}_i^{[0]}$, $\delta$; $\tilde{\mathcal{Y}} \gets \emptyset$
			
			\For {$\ell = 1: 1: \left| \mathcal{I} \right|$}
			
			\State  $\mathbf{y}^{[L]}_i \gets F(\mathbf{x}^{[0]}_i)$
			
			\State  $\epsilon_F(\mathbf{x}_i^{[0]},\delta) \gets $ \textsc{MaxSensitivity}$(F,\mathbf{x}_i^{[0]},\delta)$
			
			\State   $\tilde{\mathcal{Y}}_i \gets \{\mathbf{y}^{[L]}\mid \left\|\mathbf{y}^{[L]}-\mathbf{y}_i^{[L]}\right\|_{\infty} \le \epsilon_F(\mathbf{x}_i^{[0]},\delta)\}$
			
			\State $\tilde{\mathcal{Y}} \gets \tilde{\mathcal{Y}} \cup \tilde{\mathcal{Y}}_i$
			\EndFor
			
			\State 	\Return $\tilde{\mathcal{Y}}$
			
			\EndFunction
		\end{algorithmic}
		
	\end{algorithm} 
	
	\begin{example}\label{example_1}
		A neural network with 2 inputs, 2 outputs and 1 hidden layer consisting of 5 neurons is considered. The activation function for the hidden layer is choosen as \texttt{tanh} function and \texttt{purelin} function is for the output layer. The weight matrices and bias vectors are randomly generated as below: 
		\begin{align*}
		&\mathbf{W}^{[1]}=\left[ {\begin{array}{*{20}c}
			-0.9507 &   -0.7680  \\
			0.9707   &   0.0270   \\
			-0.6876  &   -0.0626   \\
			0.4301   &   0.1724  \\
			0.7408  &  -0.7948   \\
			\end{array} } \right],~\boldsymbol{\theta}^{[1]}=\left[ {\begin{array}{*{20}c}
			1.1836 \\
			-0.9087 \\
			-0.3463 \\
			0.2626 \\
			-0.6768 \\
			\end{array} } \right]
		\\
		&\mathbf{W}^{[2]}=\left[ {\begin{array}{*{20}c}
			0.8280 &   0.6839  &  1.0645 &  -0.0302  &  1.7372 \\
			1.4436  &  0.0824 &  0.8721  &  0.1490 &  -1.9154 \\
			\end{array} } \right]
		\\
		&\boldsymbol{\theta}^{[2]}=\left[ {\begin{array}{*{20}c}
			-1.4048 \\
			-0.4827 \\
			\end{array} } \right] .
		\end{align*}  
		
		The input set is considered as 	$\mathcal{X}_1=\{[x_1~x_2]^{\top} \mid \left|x_1-0.5\right| \le 0.5 \wedge \left|x_2-0.5\right| \le 0.5\}$. In order to execute function \texttt{OutputReach} described in Algorithm \ref{algorithm_2}, the first step is to initialize the lattices $\mathcal{L}_i$ with centers $\mathbf{x}_i^{[0]}$ and radius $\delta$. In this example, the radius is chosen to be $0.1$ and $25$ lattices are generated as in  Fig. \ref{lattice} shown in gray, which means there are in total $25$ simulations to be executed for the output reachable set estimation. 
		
		\begin{figure}
			\begin{center}
				\includegraphics[width=5cm]{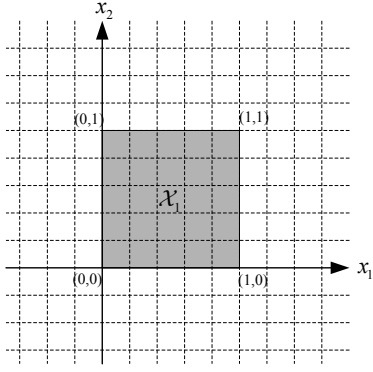}
				\caption{Input sets $\mathcal{X}_1$ and $25$ lattices with radius of $\delta=0.1$.}
				\label{lattice}
			\end{center}
		\end{figure}
		
		Executing function \texttt{OutputReach} for $\mathcal{X}_1$, the estimated output reachable set is given in Fig. \ref{reach_1}, in which 25 reachtubes are obtained and the union of them is an over-approximation of reachable set $\mathcal{Y}$. To validate the result, $10000$ random outputs are generated, it is clear to see that all the outputs are included in the estimated reachable set, showing the effectiveness of the proposed approach. 
		
		\begin{figure}
			\begin{center}
				\includegraphics[width=9cm]{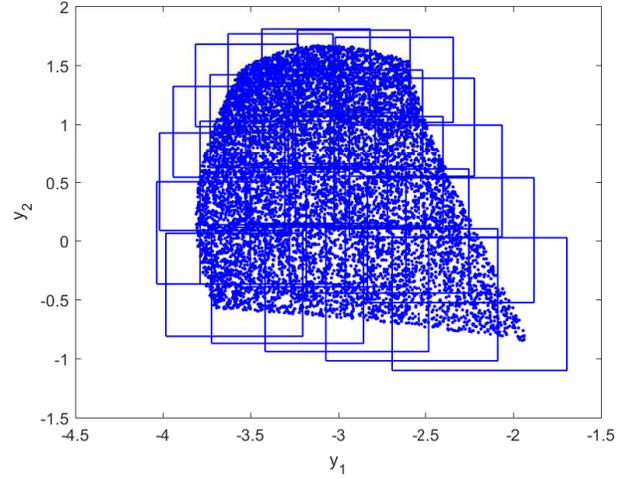}
				\caption{Output reachable set estimation with input set $\mathcal{X}_1$ and $\delta=0.1$. 25 reachtubes are computed and $10000$  random outputs are all included in the estimated reachable set. }
				\label{reach_1}
			\end{center}
		\end{figure}
		
		Moreover, we choose different radius for discretizing state space to show how the choice of radius affects the estimation outcome. 
		As  mentioned before,  a smaller radius implies a tighter approximation of input sets and is supposed to achieve a preciser estimation. Here, we select the radius as $\delta \in \{0.1,0.05,0.025,0.0125\}$. With finer discretizations, more simulations are required for running function \texttt{OutputReach}, but  tighter estimations for the output reachable set can be obtained. The output reachable set estimations are shown in Fig. \ref{reach_2}. Comparing those results, it can be observed that a smaller radius can lead to a better estimation result at the expense of more simulations and computation time, as shown in Table \ref{tab_2}. 
		\begin{figure}
			\begin{center}
				\includegraphics[width=9cm]{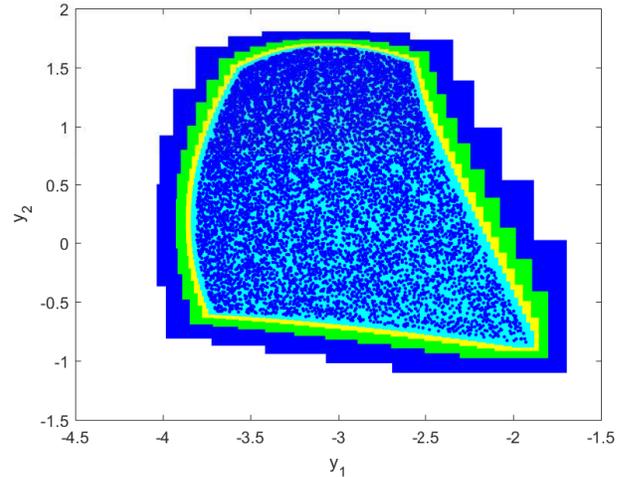}
				\caption{Output reachable set estimations with input set $\mathcal{X}_1$ and $\delta=0.1 (\mathrm{blue}), 0.05 (\mathrm{green}), 0.025 (\mathrm{yellow}),0.0125 (\mathrm{cyan})$. Tighter estimations are obtained with smaller radius. }
				\label{reach_2}
			\end{center}
		\end{figure}
		
		\begin{table}[t!]
			\centering
			\caption{Comparison of output reachable set estimations with different radius}\label{tab_2}
			\begin{tabular}{c||c||c||c||c}
				\hline
				Radius & $0.1$ & $0.05$ & $0.025$ & $0.0125$ \\
				\hline
				Computation time  & $0.044$ s & $0.053$ s &  $0.086$ s & $0.251$ s \\
				\hline
				Simulations & 25 & 100 & 400 & 1600  \\
				\hline
			\end{tabular}
		\end{table}
	\end{example}
	
	Algorithm \ref{algorithm_2} is sufficient to solve the output reachable set estimation problem for an MLP, that is Problem \ref{problem1}. Then, we can move forward to Problem \ref{problem2}, the safety verification problem for an MLP with a given safety specification $\mathcal{S}$. 
	
	\begin{proposition}
		Consider an MLP in the form of (\ref{NN}), an output reachable set estimation and a safety specification $\mathcal{S}$, the MLP is safe if  $\tilde{\mathcal{Y}} \cap \neg \mathcal{S} = \emptyset$, where $\tilde{\mathcal{Y}}$ is the estimated output reachable set obtained by Algorithm \ref{algorithm_2}. 
	\end{proposition}
	\begin{proof}
		By Algorithm \ref{algorithm_2}, we have $ \mathcal{Y} \subseteq \tilde{\mathcal{Y}} $, where $\mathcal{Y}$ is the actual output reachable set of the MLP. Using Lemma \ref{lemma1}, the safety can be guaranteed. The proof is complete.
	\end{proof}
	
	The simulation-based safety verification algorithm is presented in Algorithm \ref{algorithm_3}. 
	\begin{algorithm}
		
		\caption{Safety Verification for MLP} \label{algorithm_3}
		
		\begin{algorithmic}[1]
			\Require MLP $F$, input set $\mathcal{X}$, safety requirement $\mathcal{S}$.
			\Ensure Safe or unsafe property. 
			\Function{SafetyVeri}{$F$,~$\mathcal{X}$,~$\mathcal{S}$}
			
			\State Initialize $\mathcal{L}_i$, $i \in \mathcal{I}$, $\mathbf{x}_i^{[0]}$, $\delta$
			
			\For {$\ell = 1: 1: \left| \mathcal{I} \right|$}
			
			\State  $\mathbf{y}^{[L]}_i \gets F(\mathbf{x}^{[0]}_i)$
			\If {$\mathbf{y}_i^{[L]} \cap  \neg \mathcal{S} \ne \emptyset$ and $\mathbf{x}_i^{[0]} \in \mathcal{X}$}
			
			\State \Return UNSAFE
			
			\EndIf
			
			\EndFor
			
			\State  $\tilde{\mathcal{Y}} \gets$ \textsc{OutputReach}($F$, $\mathcal{X}$)
			
			\If {$\tilde{\mathcal{Y}} \cap \neg \mathcal{S} = \emptyset$}
			\State \Return SAFE
			\Else
			\State \Return UNCERTAIN
			\EndIf		
			\EndFunction
		\end{algorithmic}
		
	\end{algorithm}

	\begin{remark}The Algorithm \ref{algorithm_3} is sound for the cases of SAFE and UNSAFE, that is, if it returns SAFE then the system is safe; when it returns UNSAFE there exists at least one output from input set is unsafe since the existence of one simulation that is unsafe is sufficient to claim unsafeness. Additionally, if it returns UNCERTAIN, caused by the fact $\tilde{\mathcal{Y}} \cap \neg \mathcal{S} \ne \emptyset$, that means the safety property is unclear for this case. 
	\end{remark}
	\begin{example}
		The same MLP as in Example \ref{example_1} is considered, and the input set is considered to be 
		$\mathcal{X}_2=\{[x_1~x_2]^{\top} \mid \left|x_1-0.5\right| \le 1.5 \wedge \left|x_2-0.5\right| \le 0.1\}$.
		Furthermore, the safe specification $\mathcal{S}$ is assumed as 
		$\mathcal{S} = \{[x_1~x_2]^{\top} \mid -3.7\le x_1 \le -1.5\}$.
		To do  safety verification, the first step of using function \texttt{SafetyVeri} in Algorithm \ref{algorithm_3} is to initialize the lattices $\mathcal{L}_i$ with two radius $\delta_1=0.1$ and  $\delta_2=0.05$.
		\begin{figure}
			\begin{center}
				\includegraphics[width=9cm]{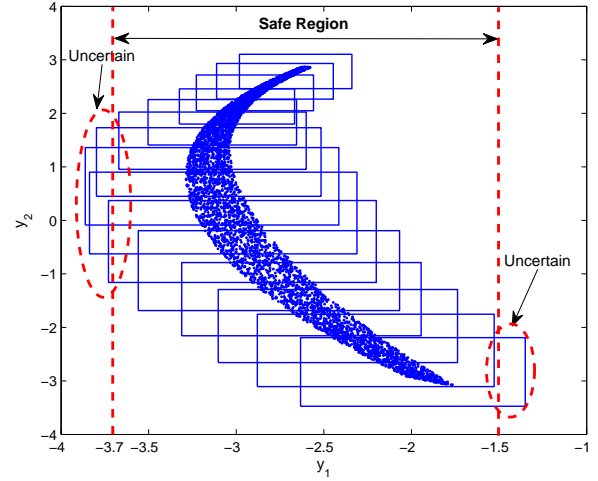}
				\caption{Safety verification for input belonging to $\mathcal{X}_2$. With radius $\delta=0.1$, the MLP cannot be concluded to be safe or not, since there exist intersections between the estimated reachable set and the unsafe region.}
				\label{veri_1}
			\end{center}
		\end{figure}
		\begin{figure}
			\begin{center}
				\includegraphics[width=9cm]{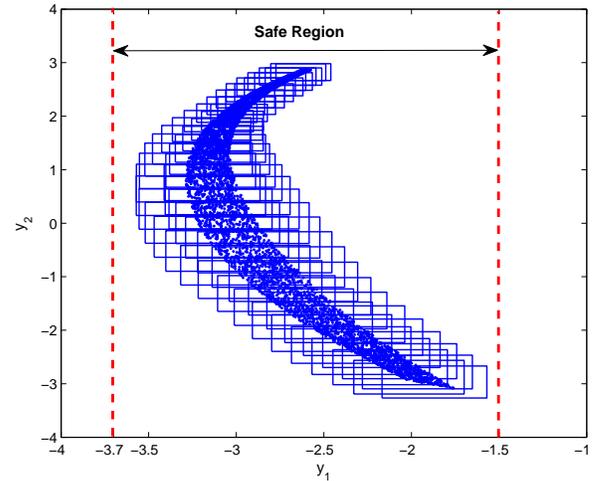}
				\caption{Safety verification for input belonging to $\mathcal{X}_2$. With  $\delta=0.05$, the safety can be confirmed, since the estimated reachable set is in the safe region.}
				\label{veri_2}
			\end{center}
		\end{figure}
		Since two radius are used, the verification results could be different due to different precisions selected. The verification results are shown in Figs. \ref{veri_1} and \ref{veri_2}, and compared in Table \ref{tab_3}.
		
		\begin{table}[ht!]
			\centering
			\caption{Comparison of safety verifications with different radius}\label{tab_3}
			\begin{tabular}{c||c||c}
				\hline
				Radius & $\delta_1=0.1$ & $\delta_2=0.05$\\
				\hline
				Safety & UNCERTAIN & SAFE   \\
				\hline
				Simulations & 15 & 60  \\
				\hline
				
			\end{tabular}
		\end{table} 
		The safety property is uncertain when the radius is chosen as $0.1$. However, we can conclude the safeness of the MLP when a smaller radius $\delta =0.05$ is chosen at the expense of increasing the number of simulations from $15$ to $60$. 
	\end{example}
	
	\section{Application in Safety Verification for Robotic Arm Models} 
	In this section, our study focuses
	on \emph{learning forward kinematics} of a robotic arm model with two joints, see Fig. \ref{robotic_arm}. The learning task of the MLP is to predict the position $(x,y)$
	of the end with knowing the joint angles $(\theta_1,\theta_2)$. 
	For the robotic arm, the input space $[0,2\pi]\times [0,2\pi]$ for $(\theta_1,\theta_2)$  is classified into three zones for its operations: 
	\begin{figure}[h!]
		\begin{center}
			\includegraphics[width=4cm]{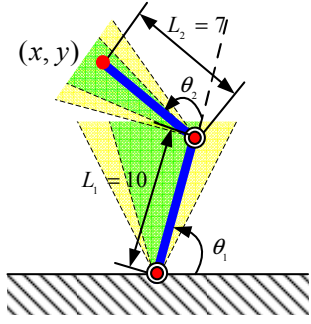}
			\caption{Robotic arm with two joints. The normal working zone of $(\theta_1,\theta_2)$ is colored in green and the buffering zone is in yellow.}
			\label{robotic_arm}
		\end{center}
	\end{figure}
	\begin{enumerate}
		\item [(1)] \emph{Normal working zone:} The normal working zone is the working region that the robotic arm works most of the time in, and the input-output training data is all collected from this region to train the MLP model. This zone is assumed to be   $\theta_1,\theta_2 \in [\frac{5\pi}{12},\frac{7\pi}{12}]$. 
		
		\item[(2)] \emph{Forbidden zone:} The forbidden zone specifies the region that the robotic arm will never operate in due to physical constraints or safety considerations in design. This zone is assumed as $\theta_1,\theta_2 \in [0,\frac{\pi}{3}] \cup [\frac{2\pi}{3},2\pi] $.
		
		\item[(3)] \emph{Buffering zone:} The buffering zone is between the normal working zone and the forbidden zone. Some occasional operations may occur out of normal working zone, but it remains in the buffering zone, not reaching the forbidden zone. This zone is   $\theta_1,\theta_2 \in [\frac{\pi}{3},\frac{5\pi}{12}] \cup [\frac{7\pi}{12},\frac{2\pi}{3}] $.
	\end{enumerate}
	
	The safety specification for the position $(x,y)$ is considered as 
	$\mathcal{S}=\{(x,y)\mid -14 \le x\le 3 \wedge 1 \le y \le 17\}$.
	In the safety point of view, the MLP needs to be verified that all the outputs produced by the inputs in the normal working zone and buffering zone will satisfy  safety specification $\mathcal{S}$. One point needs to emphasize is that the MLP is trained by the data in normal working zone, but the safety specification is defined on both normal working zone and buffering zone. 
	
	Using the data from normal working zone, the learning process is standard by using \texttt{train} function in the neural network toolbox in Matlab. The MLP considered for this example is with 2 inputs, 2 outputs and 1 hidden layer consisting of 5 neurons. The activation functions \texttt{tanh} and \texttt{purelin} are for hidden layer and output layer, respectively.  However, for the trained MLP, there is no safety assurance for any  manipulations, especially for the ones in the buffering zone where no input-output data is used to train the MLP. To verify the safety specification of the trained MLP, our function \texttt{SafetyVeri} presented in Algorithm \ref{algorithm_3} is used for this example. 
	
	First, we train the MLP with inputs $(\theta_1,\theta_2) \in [\frac{5\pi}{12},\frac{7\pi}{12}] \times [\frac{5\pi}{12},\frac{7\pi}{12}]$ along with their corresponding  outputs. Then, to use function \texttt{SafetyVeri} for the inputs in both normal working zone and buffering zone, we discretize input space $[\frac{\pi}{3},\frac{2\pi}{3}] \times [\frac{\pi}{3},\frac{2\pi}{3}]$ with radius $\delta=0.05$. The safety verification result is shown in Fig. \ref{robotic_veri_0.05}. It can be observed that the safety property of the MLP is uncertain since the estimated reachable set reaches out of the safe region $\mathcal{S}$. Then, 5000 random simulations are executed, and it shows that no output is unsafe. However, 5000 simulations or any finite number of simulations are not sufficient to say the MLP is safe. Therefore, to soundly claim that the MLP trained with the data collected in normal working zone is safe with regard to both normal working and buffering zones, a smaller radius $\delta =0.02$ has to be adopted. The verification result with $\delta =0.02$ is shown in Fig. \ref{robotic_veri_0.02}. It can be seen that the reachable set of the MLP is contained in the safe region, which is sufficient to claim the safety of the robotic arm MLP model. 
	
	\begin{figure}
		\begin{center}
			\includegraphics[width=9cm]{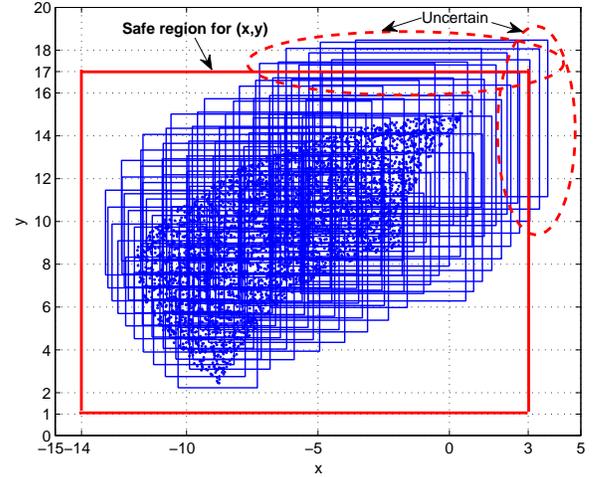}
			\caption{Safety verification for MLP model of robotic arm with two joints. With radius $\delta=0.05$, the safety cannot be determined.}
			\label{robotic_veri_0.05}
		\end{center}
	\end{figure}
	\begin{figure}
		\begin{center}
			\includegraphics[width=9cm]{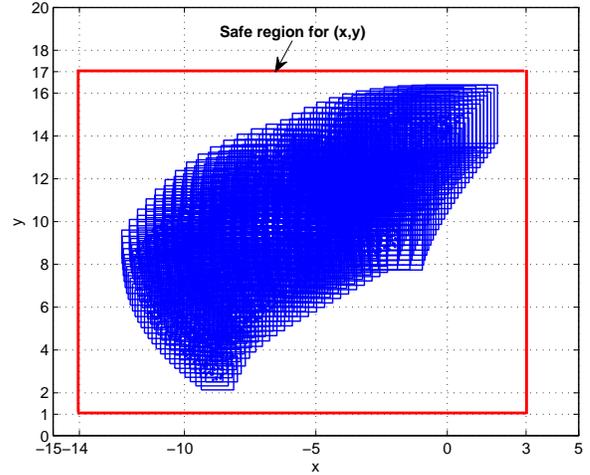}
			\caption{Safety verification for MLP model of robotic arm with two joints. With radius $\delta=0.02$, the MLP model can be claimed to be safe.}
			\label{robotic_veri_0.02}
		\end{center}
	\end{figure}
	
	\section{Related Work}
	In \cite{katz2017reluplex}, an SMT solver named Reluplex is proposed for a special class of neural networks with ReLU activation functions. The Reluplex extends the well-known Simplex algorithm from linear functions to ReLU functions by making use of the piecewise linear feature of ReLU functions. In contrast to Reluplex solver in the context of SMT, the approach developed in our paper aims at the reachability problem of neural networks. In \cite{xiang2017reachable_arxiv}, A layer-by-layer approach is developed for the output reachable set computation of ReLU neural networks. The computation is formulated in the form of a set of manipulations for a union of polytopes.  It should be noted that our approach is general in the sense that it is not tailored for a specific activation function.
	
	The authors of  \cite{pulina2012challenging} and \cite{pulina2010abstraction} propose an approach for verifying properties of neural networks with sigmoid activation functions. They replace the activation functions with piecewise linear approximations thereof, and then invoke black-box SMT
	solvers. Our approach does not use any approximations of activation functions. Instead, the approximation of our approach comes from the over-approximation of output set of each neuron, by lower and upper bounds.
	
	In a recent paper \cite{huang2017safety}, the authors propose a method for proving the local adversarial robustness of neural networks. The purpose of paper \cite{huang2017safety} is to check the robustness around one fixed point. Instead, our approach is focusing on a set defined on a continuous domain, rather than one single point.  
	
	Finally, Lyapunov function approach plays a key role in stability and reachability analysis for dynamical systems such as uncertain systems \cite{xiang2018parameter}, positive systems \cite{xiang2017stability}, hybrid systems \cite{xiang2016necessary,xiang2017output,xiang2017reachable}. The neural networks involved in papers \cite{xu2017reachable,zuo2014non,thuan2016reachable} are recurrent neural networks modeled by a family of differential equations so that Lyapunov function approach works. For the MLP considered in this paper which is described  by a set of nonlinear algebraic equations, Lyapunov function approach is not a suitable tool. Thus, we introduce another conception called maximal sensitivity to characterize the reachability property of neural networks.
	
	\section{Conclusions}
	This paper represents a simulation-based method to compute the output reachable sets for MLPs by solving a chain of optimization problems. Based on the monotonic assumption which can be satisfied by a variety of activation functions of neural networks, the computation for the so-called maximum sensitivity is formulated as a set of optimization problems essentially described as convex optimizations. Then, utilizing the results for maximum sensitivity, the reachable set estimation of an MLP can be performed by checking the maximum sensitivity property of finite number of sampled inputs to the MLP. Finally, the safety property of an MLP can be verified based on the estimation of output reachable set. Numerical examples are provided to validate the proposed algorithms. Finally, an application of safety verification for a robotic arm model with two joints is presented.
	
	\bibliographystyle{ieeetr}
	\bibliography{ref}
\end{document}